\documentclass[twocolumn]{article}

\usepackage{amsmath,amsthm,amssymb,amsfonts}
\usepackage{verbatim}

\usepackage{stmaryrd} %for \leftrightarroweq
\usepackage{hyperref}

\usepackage[colorinlistoftodos]{todonotes}

\usepackage{tikz}
\usepackage{tkz-berge, tkz-graph}
\tikzset{vertex/.style={circle,draw,fill,inner sep=0pt,minimum size=1mm}}

\theoremstyle{plain}
\newtheorem{thm}{Theorem}

\newtheorem{prop}[thm]{Proposition}
\newtheorem{cor}[thm]{Corollary}

\theoremstyle{definition}
\newtheorem{definition}[thm]{Definition}
\newtheorem{exl}[thm]{Example}

\numberwithin{thm}{section}

\def\Z{{\mathbb Z}}

\begin{document}
\title{Connectivity Preserving Multivalued Functions in Digital Topology}
\author{Laurence Boxer
         \thanks{
    Department of Computer and Information Sciences,
    Niagara University,
    Niagara University, NY 14109, USA;
    and Department of Computer Science and Engineering,
    State University of New York at Buffalo.
    E-mail: boxer@niagara.edu
    }
\and
{P. Christopher Staecker
\thanks{
 Department of Mathematics,
 Fairfield University,
 Fairfield, CT 06823-5195, USA.
 E-mail: cstaecker@fairfield.edu
}
}
}
\date{ }
\maketitle

\begin{abstract}
We study {\em connectivity preserving multivalued functions}
~\cite{Kovalevsky}
between digital images. This notion generalizes that
of {\em continuous multivalued functions}
~\cite{egs08,egs12} studied mostly in the setting of the digital plane $\Z^2$. 
We show that connectivity preserving
multivalued functions, like continuous multivalued functions, 
are appropriate models for digital morpholological operations. 
Connectivity preservation, unlike continuity, is preserved by compositions, and generalizes easily 
to higher dimensions and arbitrary adjacency relations.

\end{abstract}

\section{Introduction}
Continuous functions between digital images were
introduced in ~\cite{Rosenfeld} and have been
explored in many subsequent papers. However, the notion
of a continuous function $f$ between digital images 
$X$ and $Y$ does
not always yield results analogous to what might
be expected from parallels with the Euclidean objects
modeled by $X$ and $Y$. For example, in Euclidean
space, if $X$ is a square and $Y$ is an arc such that
$Y \subset X$, then $Y$ is a continuous retract of 
$X$~\cite{Borsuk}. However,~\cite{Boxer94} gives an 
example of a digital square $X$ containing a digital
arc $Y$ such that $Y$ is not a continuous retract
of $X$.

In order to address such anomalies, digitally
continuous multivalued functions were introduced
~\cite{egs08,egs12}. These papers showed that in
some ways, digitally continuous multivalued functions
allow the digital world to model the Euclidean world
better than digitally
continuous single-valued functions. However, digitally
continuous multivalued functions have their own
anomalies, e.g., composition does not always 
preserve continuity among digitally
continuous multivalued functions~\cite{gs15}.

In this paper, we study connectivity preserving
multivalued functions between digital images and show that these
offer some advantages over
continuous multivalued functions. One of these
advantages is that the composition of connectivity
preserving multivalued functions between digital
images is connectivity preserving. Another advantage is that the concept of connectivity preservation of a map on a digital image can be defined without any reference to a particular realization of $X$ as a subset of $\Z^n$; by contrast, an example discussed in Section~\ref{prelims} shows that continuity of a multivalued map on $(X,\kappa)$ is heavily influenced by how $X$ is embedded in $\Z^n$. These
advantages help us to generalize easily 
our definitions and results to images of any dimension and adjacency relations.

There are also disadvantages in the use of connectivity preserving 
multivalued functions as compared with the use of continous
multivalued functions. In section~\ref{retract-section}, 
we show ways in which continuous multivalued functions 
better model retractions of Euclidean topology than
do connectivity preserving multivalued functions.

\section{Preliminaries}
\label{prelims}
We will assume familiarity with the topological theory of digital images. See, e.g., \cite{Boxer94} for the standard definitions. All digital images $X$ are assumed to carry their own adjacency relations (which may differ from one image to another). When we wish to emphasize the particular adjacency relation we write the image as $(X,\kappa)$, where $\kappa$ represents
the adjacency relation.

Among the commonly used adjacencies are the $c_u$-adjacencies.
Let $x,y \in \Z^n$, $x \neq y$. Let $u$ be an integer,
$1 \leq u \leq n$. We say $x$ and $y$ are $c_u$-adjacent if
\begin{itemize}
\item There are at most $u$ indices $i$ for which 
      $|x_i - y_i| = 1$.
\item For all indices $j$ such that $|x_j - y_j| \neq 1$ we
      have $x_j=y_j$.
\end{itemize}
We often label a $c_u$-adjacency by the number of points
adjacent to a given point in $\Z^n$ using this adjacency.
E.g.,
\begin{itemize}
\item In $\Z^1$, $c_1$-adjacency is 2-adjacency.
\item In $\Z^2$, $c_1$-adjacency is 4-adjacency and
      $c_2$-adjacency is 8-adjacency.
\item In $\Z^3$, $c_1$-adjacency is 6-adjacency,
      $c_2$-adjacency is 18-adjacency, and $c_3$-adjacency
      is 26-adjacency.
\end{itemize}

For much of the paper, we will not need to assume that $(X,\kappa)$ is embedded as a subset of $(\Z^n, \kappa)$ for some particular $n$.

A subset $Y$ of a digital image $(X,\kappa)$ is
{\em $\kappa$-connected}~\cite{Rosenfeld},
or {\em connected} when $\kappa$
is understood, if for every pair of points $a,b \in Y$ there
exists a sequence $\{y_i\}_{i=0}^m \subset Y$ such that
$a=y_0$, $b=y_m$, and $y_i$ and $y_{i+1}$ are 
$\kappa$-adjacent for $0 \leq i < m$.
The following generalizes a definition of
~\cite{Rosenfeld}.

\begin{definition}\label{continuous}
{\rm ~\cite{Boxer99}}
Let $(X,\kappa)$ and $(Y,\lambda)$ be digital images. A function
$f: X \rightarrow Y$ is $(\kappa,\lambda)$-continuous if for
every $\kappa$-connected $A \subset X$ we have that
$f(A)$ is a $\lambda$-connected subset of $Y$. 
\end{definition}

When the adjacency relations are understood, we will simply say that $f$ is \emph{continuous}. Continuity can be reformulated in terms of adjacency of points:
\begin{thm}
{\rm ~\cite{Rosenfeld,Boxer99}}
A function $f:X\to Y$ is continuous if and only if, for any adjacent points $x,x'\in X$, the points $f(x)$ and $f(x')$ are equal or adjacent. \qed
\end{thm}

For two subsets $A,B\subset X$, we will say that $A$ and $B$ are \emph{adjacent} when there exist points $a\in A$ and $b\in B$ such that $a$ and $b$ are equal or adjacent. Thus sets with nonempty intersection are automatically adjacent, while disjoint sets may or may not be adjacent. It is easy to see that a union of connected adjacent sets is connected. 

A \emph{multivalued function} $f:X\to Y$ assigns a subset of $Y$ to each point of $x$. We will  write $f:X \multimap Y$. For $A \subset X$ and a multivalued function $f:X\multimap Y$, let $f(A) = \bigcup_{x \in a} f(x)$. 

\begin{definition}
\label{mildly}
\rm{\cite{Kovalevsky}}
A multivalued function $f:X\multimap Y$ is \emph{connectivity preserving} if $f(A)\subset Y$ is connected whenever $A\subset X$ is connected.
\end{definition}

As is the case with Definition \ref{continuous}, we can reformulate connectivity preservation in terms of adjacencies.

\begin{thm}
\label{mildadj}
A multivalued function $f:X \multimap Y$ is \emph{connectivity preserving} if and only if the following are satisfied:
\begin{itemize}
\item For every $x \in X$, $f(x)$ is a connected subset of $Y$.
\item For any adjacent points $x,x'\in X$, the sets $f(x)$ and $f(x')$ are adjacent.
\end{itemize}
\end{thm}

\begin{proof}
First assume that $f$ satisfies the two conditions above, let $A$ be connected, and we will show that $f(A)$ is connected. Take two points $y,y' \in f(A)$, and we will find a connected subset $B \subset f(A)$ containing $y$ and $y'$, and thus $y$ and $y'$ are connected by a path in $f(A)$. Since $y,y' \in f(A)$, there are points $x,x'\in A$ with $y\in f(x)$ and $y'\in f(x')$. Since $A$ is connected there is a path $x=x_0, x_1,\dots, x_k=x'$ with $x_i\in A$ and $x_i$ adjacent to $x_{i+1}$ for each $i$.

By our hypotheses, we have $f(x_i)$ connected and $f(x_i)$ adjacent to $f(x_{i+1})$ for each $i$. Thus the union 
\[ B = \bigcup_{i=0}^k f(x_i) \]
is connected, since it is a union of connected adjacent sets. So $B \subset f(A)$ is connected and contains $y$ and $y'$, which concludes the proof that $f(A)$ is connected.

Now for the converse assume that $f$ is connectivity preserving, and we will prove the two properties in the statement of the theorem. The first property is trivially satisfied since $f(x) = f(\{x\})$ and $\{x\}$ is connected. To prove the second property, assume that $x,x'\in X$ are adjacent, and we will show that $f(x)$ and $f(x')$ are  adjacent. 

Since $x$ and $x'$ are adjacent, the set $\{x,x'\}$ is connected and thus the set 
$f(\{x,x'\}) = f(x) \cup f(x')$ is connected. Therefore, $f(x)$ must be adjacent to $f(x')$.
\end{proof}

Definition~\ref{mildly} is related to a definition of multivalued continuity for subsets of $\Z^n$ given and explored by Escribano, Giraldo, and Sastre in \cite{egs08, egs12} based on subdivisions. (These papers make a small error with respect to compositions, which is corrected in \cite{gs15}.) Their definitions are as follows:
\begin{definition}
For any positive integer $r$, the \emph{$r$-th subdivision} of $\Z^n$ is
\[ \Z_r^n = \{ (z_1/r, \dots, z_n/r) \mid z_i \in \Z \}. \]
An adjacency relation $\kappa$ on $\Z^n$ naturally induces an adjacency relation (which we also call $\kappa$) on $\Z_r^n$ as follows: $(z_1/r, \dots, z_n/r), (z'_1/r, \dots, z'_n/r)$ are adjacent in $\Z^n_r$ if and only if $(z_1, \dots, z_n)$ and $(z_1, \dots, z_n)$ are adjacent in $\Z^n$.

Given a digital image $(X,\kappa) \subset (\Z^n,\kappa)$, the \emph{$r$-th subdivision} of $X$ is 
\[ S(X,r) = \{ (x_1,\dots, x_n) \in \Z^n_r \mid (\lfloor x_1 \rfloor, \dots, \lfloor x_n \rfloor) \in X \}. \]

Let $E_r:S(X,r) \to X$ be the natural map sending $(x_1,\dots,x_n) \in S(X,r)$ to $(\lfloor x_1 \rfloor, \dots, \lfloor x_n \rfloor)$. 

For a digital image $(X,\kappa) \subset (\Z^n,\kappa)$, a function $f:S(X,r) \to Y$ \emph{induces a multivalued function $F:X\multimap Y$} as follows:
\[ F(x) = \bigcup_{x' \in E^{-1}_r(x)} \{f(x')\}. \]

A multivalued function $F:X\multimap Y$ is called \emph{continuous} when there is some $r$ such that $F$ is induced by some single valued continuous function $f:S(X,r) \to Y$. 
\end{definition}

An example of two spaces and their subdivisions is given in Figure \ref{subdivfig}.

\begin{figure}
\begin{center}
\begin{tabular}{cccc}
\begin{tikzpicture}[scale=.4]
\foreach \x/\y in {1/0,0/1} {
	\filldraw[fill=gray, xshift=2*\x cm,yshift=2*\y cm]
		(45:1.2) \foreach \t in {135,225,315,45} { -- (\t:1.2) };
}
\end{tikzpicture}\qquad
&
\begin{tikzpicture}[scale=.2]
\foreach \x/\y in {2/0,2/1,3/0,3/1,0/2,0/3,1/2,1/3} {
	\filldraw[fill=gray, xshift=2*\x cm,yshift=2*\y cm]
		(45:1.2) \foreach \t in {135,225,315,45} { -- (\t:1.2) };
}
\end{tikzpicture}\qquad
&
\begin{tikzpicture}[scale=.4]
\foreach \x/\y in {0/0,0/1} {
	\filldraw[fill=gray, xshift=2*\x cm,yshift=2*\y cm]
		(45:1.2) \foreach \t in {135,225,315,45} { -- (\t:1.2) };
}
\end{tikzpicture}\qquad
&
\begin{tikzpicture}[scale=.2]
\foreach \x/\y in {0/0,0/1,1/0,1/1,0/2,0/3,1/2,1/3} {
	\filldraw[fill=gray, xshift=2*\x cm,yshift=2*\y cm]
		(45:1.2) \foreach \t in {135,225,315,45} { -- (\t:1.2) };
}
\end{tikzpicture}
\\
$X$ & $S(X,2)$ & $Y$ & $S(Y,2)$
\end{tabular}
\end{center}
\caption{Two images $X$ and $Y$ with their second subdivisions. %(Subdivisions are drawn at half-scale.)
\label{subdivfig}}
\end{figure}
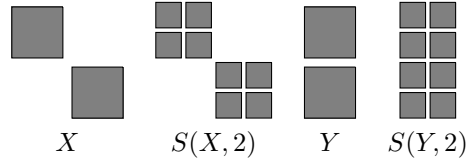

Note that the subdivision construction (and thus the notion of continuity) depends on the particular embedding of $X$ as a subset of $\Z^n$. In particular we may have $X, Y \subset \Z^n$ with $X$ isomorphic to $Y$ but $S(X,r)$ not isomorphic to $S(Y,r)$. This in fact is the case for the two images in Figure \ref{subdivfig}, when we use 8-adjacency for all images. The spaces $X$ and $Y$ in the figure are isomorphic, each being a set of two adjacent points. But $S(X,2)$ and $S(Y,2)$ are not isomorphic since $S(X,2)$ can be disconnected by removing a single point, while this is impossible in $S(Y,2)$. 

The definition of connectivity preservation makes no reference to $X$ as being embedded inside of any particular integer lattice $\Z^n$.

\begin{prop}
\label{pt-images-connected}
\rm{\cite{egs08,egs12}}
Let $F:X\multimap Y$ be a continuous multivalued function
between digital images. Then
\begin{itemize}
\item for all $x \in X$, $F(x)$ is connected; and
\item for all connected subsets $A$ of $X$, $F(A)$ is connected.
\qed
\end{itemize}
\end{prop}

\begin{thm}
\label{cont-hierarchy}
For $(X,\kappa) \subset (\Z^n,\kappa)$, if $F:X\multimap Y$ 
is a continuous multivalued function, then $F$ is connectivity preserving.
\end{thm}
\begin{proof}
By Proposition~\ref{pt-images-connected}, for all connected subsets $A$ of $X$, $F(A)$ is connected. The assertion follows
from Definition~\ref{mildly}.
\end{proof}

The subdivision machinery often makes it difficult to prove that a given multivalued function is continuous. By contrast, many maps can easily be shown to be connectivity preserving. 

\begin{prop}
\label{1-to-all}
Let $X$ and $Y$ be digital images.
Suppose $Y$ is connected. Then the
multivalued function $f: X \multimap Y$ defined by
$f(x)=Y$ for all $x \in X$ is connectivity preserving.
\end{prop}

\begin{proof}
This follows easily from Definition~\ref{mildly}.
\end{proof}

\begin{prop}
\label{finite-to-infinite}
Let $F: (X,\kappa) \multimap (Y,\lambda)$ be a multivalued
surjection between digital images $(X,\kappa),(Y,\kappa)\subset (\Z^n, \kappa)$. If $X$ is finite and $Y$
is infinite, then
$F$ is not continuous.
\end{prop}

\begin{proof}
Since $F$ is a surjection, $X$ is finite, and $Y$
is infinite, there exists $x' \in X$ such that
$F(x')$ is an infinite set. Therefore, no continuous
single-valued function $f: S(X,r) \rightarrow Y$ induces 
$F$, since for such a function,
$\bigcup_{x \in E_r^{-1}(x')} \{f(x) \}$
is finite.
\end{proof}

\begin{cor}
Let $F: X \multimap Y$ be the
multivalued function
between digital images defined by
$F(x)=Y$ for all $x \in X$. If $X$ is finite and $Y$
is infinite and connected, then
$F$ is connectivity preserving but not continuous.
\end{cor}

\begin{proof}
This follows from Propositions~\ref{1-to-all} 
and~\ref{finite-to-infinite}.
\end{proof}

Examples of connectivity preserving but not continuous multivalued functions on finite spaces are harder to construct, since one must show that a given connectivity preserving map $X \multimap Y$ cannot be induced by any map on any subdivision. After some more development we will give such an example in Example \ref{review-square-boundary}.

Other terminology we use includes the following.
Given a digital image $(X,\kappa) \subset \Z^n$ and $x \in X$, the set of points adjacent to $x \in \Z^n$, the
neighborhood of $x$ in $\Z^n$, and the boundary of $X$
in $\Z^n$ are, respectively,
\[N_{\kappa}(x) = \{y \in \Z^n \, | \, y \mbox{ is }
    \kappa\mbox{-adjacent to }x\},\]
\[N_{\kappa}^*(x) = N_{\kappa}(x) \cup \{x\},
\]
and
\[ \delta_{\kappa}(X)=\{y \in X \, | \, 
    N_{\kappa}(y) \setminus X \neq \emptyset \}.
\]

\section{Other notions of multivalued continuity}
Other notions of continuity have been given
for multivalued functions between graphs (equivalently,
between digital images). We have the following.

\begin{definition}
\rm{~\cite{Tsaur}}
\label{Tsaur-def}
Let $F: X \multimap Y$ be a multivalued function between
digital images.
\begin{itemize}
\item $F$ has {\em weak continuity} if for each pair of
      adjacent $x,y \in X$, $f(x)$ and $f(y)$ are adjacent
      subsets of $Y$.
\item $F$ has {\em strong continuity} if for each pair of
      adjacent $x,y \in X$, every point of $f(x)$ is adjacent
      or equal to some point of $f(y)$ and every point of 
      $f(y)$ is adjacent or equal to some point of $f(x)$.
     \qed
\end{itemize}
\end{definition}

\begin{prop}
\label{mild-and-weak}
Let $F: X \multimap Y$ be a multivalued function between
digital images. Then $F$ is connectivity preserving if and
only if $F$ has weak continuity and for all $x \in X$,
$F(x)$ is connected.
\end{prop}

\begin{proof}
This follows from Theorem~\ref{mildadj}.
\end{proof}

\begin{exl}
\label{pt-images-discon}
If $F: [0,1]_{\Z} \multimap [0,2]_{\Z}$ is defined by
$F(0)=\{0,2\}$, $F(1)=\{1\}$, then $F$ has both weak and
strong continuity. Thus a multivalued function that has weak or strong continuity need not
have connected point-images. By Theorem~\ref{mildadj} and
Proposition~\ref{pt-images-connected} it
follows that neither having weak continuity nor having
strong continuity implies that a multivalued function is
connectivity preserving or continuous.
$\Box$
\end{exl}

\begin{exl}
Let $F: [0,1]_{\Z} \multimap [0,2]_{\Z}$ be defined by
$F(0)=\{0,1\}$, $F(1)=\{2\}$. Then $F$ is continuous and
has weak continuity but
does not have strong continuity. $\Box$
\end{exl}

\begin{prop}
Let $F: X \multimap Y$ be a multivalued function between
digital images. If $F$ has strong continuity and for
each $x \in X$, $F(x)$ is connected, then $F$ is
connectivity preserving.
\end{prop}

\begin{proof}
The assertion follows from Definition~\ref{Tsaur-def} and
Theorem~\ref{mildadj}. Alternately, it follows from
Proposition~\ref{mild-and-weak}, since strong continuity implies
weak continuity.
\end{proof}

The following shows that not requiring the images of
points to be connected yields topologically unsatisfying 
consequences for weak and strong continuity.

\begin{exl}
Let $X$ and $Y$ be nonempty digital images. Let
the multivalued function $f: X \multimap Y$ be defined by
$f(x)=Y$ for all $x \in X$.
\begin{itemize}
\item $f$ has both weak and strong continuity.
\item $f$ is connectivity preserving if and only if $Y$ is
      connected.
\end{itemize}
\end{exl}

\begin{proof} That $f$ has both weak and strong
continuity is clear from Definition~\ref{Tsaur-def}.

Suppose $f$ is connectivity preserving. Then for $x \in X$,
$f(x)=Y$ is connected. Conversely, if $Y$ is connected,
it follows easily from Definition~\ref{mildly} that
$f$ is connectivity preserving.
\end{proof}

As a specific example consider $X= \{0\} \subset \Z$ and $Y = \{0,2\}$, all with $c_1$ adjacency. Then the function $F:X \multimap Y$ with $F(0) = Y$ has both weak and strong continuity, even though it maps a connected image surjectively onto a disconnected image.

\section{Composition}
Connectivity preservation of multivalued functions is preserved by compositions. For two multivalued functions $f:X\multimap Y$ and $g:Y \multimap Z$, let $g \circ f:X \multimap Z$ be defined by 
\[ g\circ f (x) = g(f(x)) = \bigcup_{y \in f(x)} g(y). \]

\begin{thm}
\label{composition-thm}
If $f:X\multimap Y$ and $g:Y \multimap Z$ are connectivity preserving, then $g\circ f:X\multimap Z$ is connectivity preserving.
\end{thm}
\begin{proof}
We must show that $g\circ f(A) = g(f(A))$ is connected whenever $A$ is connected. Since $f$ is connectivity preserving we have $f(A)$ connected, and then since $g$ is connectivity preserving we have $g(f(A))$ connected.
\end{proof}

By contrast with Theorem~\ref{composition-thm}, Remark 4 of
~\cite{gs15} shows that composition does not always
preserve continuity in multivalued functions between 
digital images. The example given there has finite
digital images $X,Y,Z$ in $\Z^2$ and 
multivalued functions $F: X \rightarrow Y$,
$G: Y \rightarrow Z$ such that $F$ is 
$(4,k)$-continuous and $G$ is $(k,k')$-continuous
for $\{k,k'\} \subset \{4,8\}$, but
$G \circ F: X \rightarrow Z$ is not 
$(4,k')$-continuous. In fact, the example
presented in ~\cite{gs15} shows that even if
$F$ is a single-valued isomorphism, $G \circ F$ need
not be a continuous multivalued function. However, by 
Theorems~\ref{cont-hierarchy} and~\ref{composition-thm},
$G \circ F$ is $(4,k')$-connectivity preserving.

\section{Shy maps and their inverses}
\begin{definition}
\label{shy-def}
\cite{Boxer05}
Let $f: X \rightarrow Y$ be a
continuous surjection of digital images. We say $f$ is
{\em shy} if
\begin{itemize}
\item for each $y \in Y$, $f^{-1}(y)$ is connected, and
\item for every $y_0,y_1 \in Y$ such that $y_0$ and $y_1$ are
      adjacent, $f^{-1}(\{y_0,y_1\})$ is
      connected. 
\end{itemize}
\end{definition}

Shy maps induce surjections on fundamental groups
~\cite{Boxer05}.
Some relationships between shy maps $f$ and their inverses
$f^{-1}$ as multivalued functions were studied in
~\cite{Boxer14}, including a restricted analog of
Theorem~\ref{shy-thm} below.
We have the following.

\begin{thm}
\label{shy-thm}
Let $f: X \to Y$ be a 
continuous surjection between digital images.
Then $f$ is shy if and only if
$f^{-1}: Y \multimap X$ is a connectivity preserving multivalued
function.
\end{thm}

\begin{proof}
This follows immediately from Theorem~\ref{mildadj} and
Definition~\ref{shy-def}.
\end{proof}

\section{Morphological operators}\label{morphologysection}
In ~\cite{egs08,egs12}, it was shown that several fundamental
operations of mathematical morphology can be performed by
using continuous multivalued functions on digital images.
In this section, we obtain similar results using
connectivity preserving multivalued functions. In order to define the morphological operators, we must assume in this section that all images $X$ under consideration are embedded in $\Z^n$ for some $n$ with a globally defined adjacency relation $\kappa$. Thus in this section we always have $(X,\kappa) \subset (\Z^n, \kappa)$. The work in \cite{egs08,egs12} focuses exclusively on $n=2$, and $\kappa$ being 4- or 8-adjacency.
Our results have the advantage of being applicable in any dimensions and using any (globally defined) adjacency relation.

\subsection{Dilation and erosion}
In the following, the use of $k=4$ or $k=8$ indicates 
4-adjacency or 8-adjacency, respectively,
in $\Z^2$.

Dilation~\cite{Soille} of a binary image can be regarded as a 
method of magnifying or swelling the image. A common method
of performing a dilation of a digital image
$(X,\kappa) \subset (\Z^n,\kappa)$
is to take the dilation 
\[ D_{\kappa}(X) = 
   \bigcup_{x \in X}N_{\kappa}^*(x). \]

\begin{thm}
\label{egs-dilation}
\rm{(\cite{egs12}; proof corrected in~\cite{gs15})}
Given $(X,k) \subset (\Z^2,k)$, the multivalued functions
$\tilde{D}_k: X \rightarrow D_k(X) \subset \Z^2$ defined by
$\tilde{D}_k(x) = N_k^*(x)$, where
$k \in \{4,8\}$, are both $(4,4)$-continuous and
$(8,8)$-continuous. \qed
\end{thm}

\begin{thm}
\label{our-dilation}
Given a digital image $(X,\kappa) \subset (\Z^n,\kappa)$,
the multivalued function
$\tilde{D}_{\kappa}: X \rightarrow D_{\kappa}(X) \subset \Z^n$
defined by $\tilde{D}_{\kappa}(x) = N_{\kappa}^*(x)$
is connectivity preserving.
\end{thm}

\begin{proof}
For every $x \in X$, $\tilde{D}_{\kappa}(x)$ is $\kappa$-connected.
Given $\kappa$-adjacent points $x,x' \in X$, we have
$x' \in \tilde{D}_{\kappa}(x)$, so $\tilde{D}_{\kappa}(x)$ and 
$\tilde{D}_{\kappa}(x')$ are $\kappa$-adjacent.
The assertion follows from Theorem~\ref{mildadj}.
\end{proof}

More general dilations are defined as follows. Let
$X \subset \Z^n$ be a digital image and let 
$B \subset \Z^n$, with the origin of $\Z^n$ a member of $B$.
We call $B$ a {\em structuring element}. Given
$x \in \Z^n$, let $t_x$ be the translation by $x$:
$t_x(y)=x+y$ for all $y\in \Z^n$ .
The {\em dilation of $X$ by $B$} is 
\[ D_B(X)= \bigcup_{x \in X}t_x(B). \]
We have the following.

\begin{thm}
\label{gen-dil}
Let $X \subset \Z^n$ be a digital image with $c_u$-adjacency
for $1 \leq u \leq n$ and let 
$B \subset \Z^n$ be a structuring element. If $B$ is
$c_u$-connected, then the multivalued dilation function
$\tilde{D}_B: X \multimap D_B(X)$ defined by
$\tilde{D}_B(x) = t_x(B)$ is connectivity preserving.
\end{thm}

\begin{proof} Since $B$ is $c_u$-connected and 
$t_x$ is continuous,
$\tilde{D}_B(x)$ is connected for all $x \in X$. If $x_0$ and 
$x_1$ are $c_u$-adjacent members of $X$ and $b \in B$, then
$x_0+b$ and $x_1+b$ are $c_u$-adjacent, so
$\tilde{D}_B(x_0)$ and $\tilde{D}_B(x_1)$ are $c_u$-adjacent. The
assertion follows from Theorem~\ref{mildadj}.
\end{proof}

Note that Theorem~\ref{gen-dil} is easily generalized to
any adjacency that is preserved by translations.

There are non-equivalent definitions of the erosion operation
in the literature. We will use the
definition of~\cite{egs12}:
the $\kappa$-erosion of $X \subset \Z^n$ is
\[ E_{\kappa}(X)=\Z^n 
   \setminus D_{\kappa}(\Z^n \setminus X).\]
In~\cite{egs12}, we find the following.

\begin{quote}
The erosion operation cannot be adequately modeled as a
digitally continuous multivalued function on the set of black
pixels since it can transform a connected set into a disconnected
set, or even delete it (for example, the erosion of a
curve is the empty set and, in general, the erosion of two
discs connected by a curve would be the disconnected union
of two smaller discs). However, since the erosion of a set
agrees with the dilation of its complement, the erosion operator
can be modeled by a continuous multivalued function
on the set of white pixels.
\end{quote}

It follows from Theorem~\ref{our-dilation} that the erosion
operator can be modeled by a connectivity preserving multivalued
function on the set of white pixels. I.e., as an analog of Corollary~\ref{their-erosion} below, we have 
Corollary~\ref{our-erosion} below. We use the notation
$\overline{E}_{\kappa}$ to suggest that the function's
image is the compliment of the erosion.

\begin{cor}
\label{their-erosion}
\rm{(\cite{egs12}; proof corrected in~\cite{gs15})}
Given $X \subset \Z^n$, the multivalued function
$\overline{E}_k: \Z^2 \setminus X \rightarrow \Z^2$
given by $\overline{E}_k(y)=N_k^*(y)$ for $y \in \Z^2 \setminus X$ is 
both $(4,4)$- and $(8,8)$-continuous, where
$k \in \{4,8\}$. \qed
\end{cor}

\begin{cor}
\label{our-erosion}
Given $(X, \kappa) \subset (\Z^n,\kappa)$,
the multivalued function
$\overline{E}_{\kappa}: \Z^n \setminus X \rightarrow \Z^n$
given by $\overline{E}_\kappa(x)=N_{\kappa}^*(x)$ is 
connectivity preserving.
\end{cor}

\begin{proof} The assertion follows as in the
proof of Theorem~\ref{our-dilation}.
\end{proof}

\subsection{Closing and opening}
Like dilation, closing (or computing the closure of) a
digital image can be regarded as a way to swell the
image.

The closure operator $C_{\kappa}$ is the
result of a dilation followed by an erosion. Since we have
defined an erosion on $X$ as a dilation on 
$\Z^n \setminus X$, we cannot say 
that $C_{\kappa}$ is a composition of a dilation and an 
erosion, since the corresponding composition
$\overline{E}_{\kappa} \circ \tilde{D}_{\kappa}$
is not generally defined. However,
from the definitions above, the closure of $X$ can
be defined as
\[ C_{\kappa}(X)= \Z^n \setminus \tilde{D}_{\kappa}(\Z^n \setminus \bigcup_{x \in X}N_{\kappa}^*(x)).
\]
This yields the following results.

\begin{thm}
\label{their-closure}
\rm{\cite{egs12}}
Given $X \subset \Z^2$, the closure operator
$C_k$ is $(k,k)$-continuous,
$k \in \{4,8\}$. \qed
\end{thm}

\begin{thm}
\label{our-closure}
Given a digital image $(X,\kappa) \subset (\Z^n,\kappa)$,
the closure operator
$C_{\kappa}$ is connectivity preserving.
\end{thm}

\begin{proof} 
Note we can define a multivalued function
$\tilde{C}_{\kappa}: X \multimap C_{\kappa}(X)$ by
\[ \tilde{C}_{\kappa}(x)= \left \{ \begin{array}{ll}
   \{x\} & \mbox{if } x \in X \setminus \delta_{\kappa}(X); \\
   N_{\kappa}^*(x) \cap C_{\kappa}(x) &
      \mbox{if } x \in \delta_{\kappa}(X).
      \end{array} \right .
\]
Since $X \subset C_{\kappa}(X)$ and each point of
$N_{\kappa}^*(x)$ is $\kappa$-adjacent or equal to $x$, it follows
that $\tilde{C}_{\kappa}(x)$ is connected for all $x \in X$.
Further, for $\kappa$-adjacent $x, x' \in X$, we have
$x \in \tilde{C}_{\kappa}(x)$ and $x' \in \tilde{C}_{\kappa}(x')$, so
$f(x)$ and $f(x')$ are adjacent. The assertion follows
from Theorem~\ref{mildadj}.
\end{proof}

We find in ~\cite{egs12} the following.
\begin{quote}
As it happens in the case of the erosion, the opening operation
(erosion composed with dilation) cannot be adequately
modeled as a digitally continuous multivalued function on
the set of black pixels (the same examples used for the erosion
also work for the opening). However, since the opening
of a set agrees with the closing of its complement
~\cite{Soille}, the
k-opening operator can be modeled by a k-continuous multivalued
function on the set of white pixels.
\end{quote}

Thus, we define an opening operator for $X$ as
the closure operator on $\Z^n \setminus X$. Corresponding to
Corollary~\ref{their-opening} below, we have
Corollary~\ref{our-opening} below.

\begin{cor}
\label{their-opening}
\rm{\cite{egs12}}
Given $X \subset \Z^2$, the $k$-opening
operation on $X$ can be modeled as a
$(4,4)$- or $(8,8)$-continuous function
$\overline{O}_k: \Z^2 \setminus X \rightarrow \Z^2$.
\qed
\end{cor}

\begin{cor}
\label{our-opening}
Given $(X,\kappa) \subset (\Z^n,\kappa)$, the $\kappa$-opening
operation on $X$ can be modeled as a
connectivity preserving function
$\overline{O}_{\kappa}: \Z^n \setminus X \rightarrow \Z^n$.
\end{cor}

\begin{proof} 
The assertion follows from Theorem~\ref{our-closure}.
\end{proof}

\section{Retractions, connectivity preserving multivalued retractions, and deletion of subsets}
\label{retract-section}
%It follows from Example~\ref{retract-example} below
%that the
%converse of Theorem~\ref{cont-hierarchy} is false.
%Commented out the above because of doubt about proof.
A continuous single-valued or multivalued function, or a connectivity preserving multivalued function, $r$, from
a set $X$ to a subset $Y$ of $X$ is called a
{\em retraction} ~\cite{Borsuk}, a 
{\em multivalued retraction},
or a {\em connectivity preserving multivalued retraction}, respectively, if $r(y)=y$ 
(respectively, $r(y)=\{y\}$) for all $y \in Y$. In this case we say
$Y$ is a {\em retract of $X$}, a 
{\em multivalued retract of $X$}, or
a {\em connectivity preserving multivalued retract of $X$}, 
respectively. It is known 
~\cite{Boxer94} that the boundary of a digital square is 
not a retract of the square. By contrast, we have
the following.

\begin{exl}
\label{retract-example}
Let $X=[-1,1]_\Z^2$. Let $Y=X \setminus \{(0,0)\}$.
Then $(Y,8)$ is a connectivity preserving multivalued retract of $(X,8)$.
%However, $(Y,8)$ is not a multivalued retract of $(X,8)$.
\end{exl}

\begin{proof}
It is easy to see that the multivalued function
$r: X \multimap Y$ given by
\[ r(x)=\left \{ \begin{array}{ll}
        Y & \mbox{if } x = (0,0); \\
        \{x\} & \mbox{if } x \in Y,
        \end{array} \right .
\]
is a connectivity preserving multivalued retraction of $(X,8)$ onto $(Y,8)$. As we will see below, $(Y,8)$ is not a multivalued retract of $(X,8)$, and thus $r$ is connectivity preserving but not continuous. 
\end{proof}

We can generalize the example given above in the following result.
The existence of connectivity preserving multivalued retractions is easily formulated in terms of connected images:
\begin{thm}
\label{setretract}
Let $X$ be connected and let $A\subset X$, $A \neq \emptyset$. Then $A$ is a connectivity preserving multivalued retract of $X$ if and only if $A$ is connected.
\end{thm}
\begin{proof}
First assume that $A$ is connected. Then define $f:X\multimap A$ by:
\[ 
f(x) = \begin{cases} \{x\} & \text{ if } x \in A, \\
A & \text{ if } x \not \in A.
\end{cases}
\]
$f$ clearly has the retraction property that $f(A)=A$ and
$f(x)=\{x\}$ for all $x \in A$. To show connectivity
preservation, let $B\subset X$ be a connected set, and we will show that $f(B)$ is connected. In the case that $B \subset A$ we have $f(B) = B$ is connected. Otherwise, $B \setminus A \neq \emptyset$
so we have $f(B) = A$ which was assumed to be connected. Thus $f$ is connectivity preserving,
so $A$ is a connectivity preserving multivalued retract of $X$ as desired.

For the converse, assume that $A$ is a connectivity preserving 
multivalued retract of $X$. Since $X$ is connected,
$A$ must be connected.
\end{proof}

Theorem~\ref{setretract} makes it easy to tell when one set is a connectivity preserving multivalued retract of another. The analogous question for continuous multivalued retracts is addressed in \cite{egs12} (corrected in \cite{gs15}), where the results are quite a bit more complicated, stated in terms of {\em simple points}, characterized by 
the following.

\begin{definition}
{\rm \cite{KR}}
Let $X \subset \Z^2$. Let 
$\{k,\overline{k}\}=\{4,8\}$.
Let $p \in X$. Then $p$ is a $k$-boundary
point of $X$ if and only if
$N_{\overline{k}}(p) \setminus X \neq \emptyset$. $\Box$
\end{definition}

\begin{thm}
\label{simple-pt}
{\rm ~\cite{rosenfeld79}}
Let $X \subset \Z^2$. Then $p \in X$ is
$k$-simple, $k \in \{4,8\}$, if and only if
$p$ is a $k$-boundary point of $X$ and the number of $k$-connected components of $N_8(p) \cap X$ that are $k$-adjacent to $p$ is equal to 1.
%the number of $k$-connected components of $X$ and $X\setminus\{p\}$ are equal, and the number of $k$-connected components of $\Z^2\setminus X$ and $(\Z^2\setminus X)\cup\{p\}$ are equal.
\qed
\end{thm}

Continuous multivalued retracts relate to simple points as follows:
\begin{thm}
\label{simpleretract}
\rm{\cite[Theorem 5]{gs15}} Let $(X,8)\subset \Z^2$ be a connected digital image, and let $p\in X$. Then $X-\{p\}$ is a continuous multivalued retract of $X$ if and only if $p$ is a simple point. \qed
\end{thm}

The requirement that $p$ be a simple point is a stronger condition than $X-\{p\}$ being connected, the condition for our Theorem \ref{setretract}. The authors of \cite{gs15} 
also obtain a similar result for 4-adjacency requiring additional hypotheses, and discuss removal of pairs of simple points. Their arguments become quite difficult and do not seem able to address removal of arbitrary subsets as in Theorem \ref{setretract}. 

Contrasting the results of Theorems \ref{setretract} and \ref{simpleretract} gives examples of maps on finite spaces that are connectivity preserving but not continuous. In particular, we have the following.

\begin{exl}
\label{review-square-boundary}
Let $X$ and $Y$ be the images in Example \ref{retract-example}.
\begin{itemize}
\item The point $(0,0)$ is not a simple point of $X$ and thus,
      $Y$ is not a continuous multivalued retract of $X$, although 
      $Y$ is a connectivity preserving multivalued retract of $X$.
\item The multivalued function $r$ of
      Example~\ref{retract-example} is connectivity preserving
      but not continuous.
\end{itemize}
\end{exl}

\begin{proof}
We saw in Example~\ref{retract-example} that
$r$ is connectivity preserving and that
$Y$ is a connectivity preserving multivalued retract of $X$.
\begin{itemize}
\item Clearly, $(0,0)$ is not a simple point of $X$.
      From Theorem~\ref{simpleretract}, $Y$ is not a 
      continuous multivalued retract of $X$.
\item Were $r$ continuous then $r$ would be a     
      multivalued retraction, contrary to  
      Theorem~\ref{simpleretract}.
\end{itemize}
\end{proof}

\section{Further remarks}
We have studied connectivity preserving multivalued
functions between digital images. This notion
generalizes continuous multivalued
functions. We have shown that composition, which 
does not preserve continuity for continuous multivalued functions, preserves connectivity preservation for multivalued functions between digital images. We have obtained a number of
results for connectivity preserving multivalued functions between digital images, concerning weak and strong continuity,
shy maps, morphological operators, and
retractions; many of our results are suggested by
analogues for continuous multivalued functions
in~\cite{egs08,egs12,gs15,Boxer14}.

\section{Acknowledgment}
We are grateful for the suggestions of the anonymous
reviewers.


\begin{thebibliography}{11}

\bibitem{Borsuk}
K. Borsuk,
{\em Theory of Retracts},
Polish Scientific Publishers, Warsaw, 1967.

\bibitem{Boxer94}
L. Boxer,
Digitally Continuous Functions,
{\em Pattern Recognition Letters} 15 (1994), 833-839.

\bibitem{Boxer99}
L. Boxer,
A Classical Construction for the Digital Fundamental Group,
{\em Pattern Recognition Letters} 10 (1999), 51-62.

\bibitem{Boxer05}
L. Boxer,
Properties of Digital Homotopy,
{\em Journal of Mathematical Imaging and Vision} 22 (2005),
19-26.

\bibitem{Boxer14}
L. Boxer,
Remarks on Digitally Continuous Multivalued Functions,
{\em Journal of Advances in Mathematics}
9 (1) (2014), 1755-1762.

\bibitem {egs08}
C. Escribano, A. Giraldo, and M. Sastre,
``Digitally Continuous Multivalued Functions,''
in \emph{Discrete Geometry for Computer Imagery}, Lecture Notes in Computer Science, v. 4992, Springer,
2008, 81--92.

\bibitem{egs12} 
C. Escribano, A. Giraldo, and M. Sastre,
``Digitally Continuous Multivalued Functions, Morphological Operations and Thinning Algorithms,''
\emph{Journal of Mathematical Imaging and Vision} 42 (2012), 76--91.

\bibitem{gs15}
A. Giraldo and M. Sastre,
On the Composition of Digitally Continuous Multivalued Functions,
{\em Journal of Mathematical Imaging and Vision}, 58 (2015), 196--209.

\bibitem{KR}
T.Y. Kong and A. Rosenfeld, eds.
{\em Topological Algorithms for Digital
Image Processing}, Elsevier, 1996.

\bibitem{Kovalevsky}
V.A. Kovalevsky,
A new concept for digital geometry,
{\em Shape in Picture},
Springer-Verlag, New York, 1994, pp. 37-51.

\bibitem{rosenfeld79}
A. Rosenfeld, 
Digital Topology, 
\emph{American Mathematical Monthly} 86 (1979), 621-630.

\bibitem{Rosenfeld}
A. Rosenfeld,
`Continuous' Functions on Digital Images,
{\em Pattern Recognition Letters} 4 (1987), 177-184.

\bibitem{Soille}
Soille, P.: Morphological operators. In: Jähne, B., et al. (eds.), Signal
Processing and Pattern Recognition. Handbook of Computer
Vision and Applications, vol. 2, pp. 627?682. Academic Press,
San Diego (1999).

\bibitem{Tsaur}
Tsaur, R., and Smyth, M.: ``Continuous" multifunctions in discrete spaces with applications to fixed point theory. In: Bertrand, G., Imiya, A., Klette, R. (eds.),
{\em Digital and Image Geometry}, 
Lecture Notes in Computer Science, vol. 2243, pp. 151-162. Springer Berlin / Heidelberg (2001), http://dx.doi.org/10.1007/3-540-45576-0 5, 10.1007/3-540-45576-0 5

\end{thebibliography}
\end{document}